\documentclass[a4paper,10pt]{article}
\usepackage[utf8]{inputenc}
 \usepackage{amsmath}
 \usepackage{amssymb}
 \usepackage{bbold}
 \usepackage{graphicx}
 \usepackage{algorithm,algorithmic}
 \usepackage{url}
 \DeclareMathOperator*{\E}{\mathbb{E}}
 \DeclareMathOperator*{\argmax}{arg\,max}
 \newcommand{\cv}[1]{\underline{#1} \,}
 \newcommand{\XX}{\mathcal{X}}
 \newtheorem{theorem}{Theorem}
 \newtheorem{proof}{Proof}
 \newtheorem{proposition}{Proposition}

\begin{document}

\title{Statistical inference with probabilistic graphical models}
\author{Angélique~Drémeau\thanks{\' Ecole Normale Sup{\' e}rieure, France}%
, Christophe~Schülke\thanks{Universit\'e Paris Diderot, France}%
  , Yingying~Xu\thanks{Tokyo Institute of Technology, Japan}%
  , Devavrat~Shah\thanks{Massachusetts Institute of Technology, USA} }

\maketitle

\begin{center}
\emph{These are notes from the lecture of Devavrat Shah
given at the autumn school ``Statistical Physics, Optimization, Inference, and Message-Passing Algorithms'', 
that took place in Les Houches, France from Monday September 30th, 2013, 
till Friday October 11th, 2013. The school was organized by Florent Krzakala from UPMC \& ENS Paris, 
Federico Ricci-Tersenghi from La Sapienza Roma, Lenka Zdeborov\'a from CEA Saclay \& CNRS, 
and Riccardo Zecchina from Politecnico Torino.}
\end{center}

 \newpage
\tableofcontents

\newpage
\section{Introduction to Graphical Models}
\index{Graphical model}

\subsection{Inference}
\index{Inference}
Consider two random variables $A$ and $B$ with a joint probability distribution $P_{A,B}$. 
From the observation of the realization of one of those variables, say $B=b$, we want to infer the one that we did not observe.
To that end, we compute the conditional probability distribution $P_{A|B}$, and use it to obtain an estimate $\hat{a}(b)$ of $a$.

To quantify how good this estimate is, we introduce the \textbf{error probability}:
\begin{align}
 P_{error} &\triangleq P(A \neq  \hat{a}(b) | B = b) \\
	  &= 1 - P(A = \hat{a}(b) | B=b),  \notag
\end{align}
and we can see from the second equality that minimizing this error probability is equivalent to the following maximization problem, called maximum a posteriori (\textbf{MAP}) problem:
\begin{equation}
 \hat{a}(b) = \argmax_a P_{A|B}(a|b).
\end{equation}
\index{MAP}

The problem of computing $P_{A|B}(a|b)$ for all $a$ given $b$ is called the marginal (\textbf{MARG}) problem.
When the number of random variables increases, the MARG problem becomes difficult, because an exponential number of combinations has to be calculated.

\textbf{Fano's inequality} provides us an information-theoretical way of gaining insight into how much information about $a$ the knowledge of $b$ can give us:
\begin{equation}
 P_{error} \geq \frac{H(A|B) -1}{{\rm log}|A|},
 \label{Fano}
\end{equation}
with
\begin{align}
 H(A|B) &= \sum_b P_B(b) H(A|B=b), \notag \\
 H(A|B=b) &= \sum_a P_{A|B}(a|b) {\rm log} \left( \frac{1}{P_{A|B}(a|b)} \right). \notag
\end{align}
Fano's inequality formalises only a theoretical bound that does not tell us how to actually make an estimation. From a practical point of view, graphical models (GM) constitute here a powerful tool allowing us to write algorithms that solve inference problems.\index{Inference}

\subsection{Graphical models}
\index{Graphical model}

\subsubsection{Directed GMs}
\index{Directed}
Consider $N$ random variables $X_1 \cdots X_N$ on a discrete alphabet $\mathcal{X}$, and their joint probability distribution $P_{X_1 \cdots X_N}$.
We can always factorize this joint distribution in the following way:
\begin{equation}
 P_{X_1 \cdots X_N} = P_{X_1} P_{X_2 | X_1} \cdots P_{X_N | X_1 \cdots X_{N-1}}
 \label{eq:factorization1}
\end{equation}
and represent this factorized form by the following directed graphical model:
\index{Directed}
\begin{figure}[h]
 \begin{center}
  \includegraphics[width=0.8\textwidth]{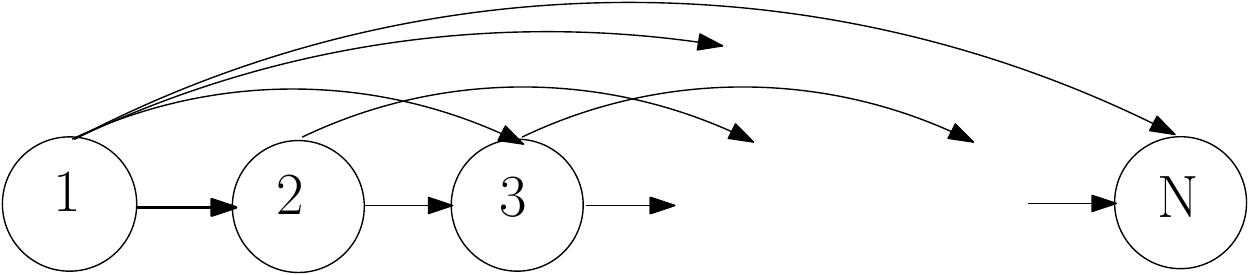}
 \end{center}
 \caption{A directed graphical model representing the factorized form (\ref{eq:factorization1}). }
\label{fig:1}
\end{figure}

In this graphical model, each node is affected to a random variable, and each directed edge represents a conditioning.
The way that we factorized the distribution, we obtain a complicated graphical model, in the sense that it has many edges.
A much simpler graphical model would be:
\begin{figure}[h]
 \begin{center}
  \includegraphics[width=0.8\textwidth]{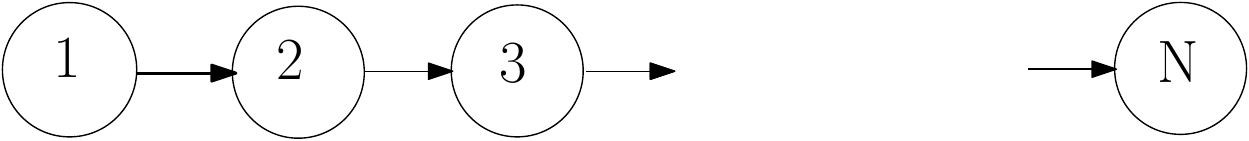}
 \end{center}
 \caption{A simpler graphical model representing the factorized form (\ref{eq:factorization2}). }
\label{fig:2}
\end{figure}

The latter graphical model corresponds to a factorization in which each of the probability distributions in the product is conditioned on only one variable:
\begin{equation}
 P_{X_1 \cdots X_N} = P_{X_1} P_{X_2 | X_1} \cdots P_{X_N | X_{N-1}} 
 \label{eq:factorization2}
\end{equation}

In the most general case, we can write a distribution represented by a directed graphical model in the factorized form:
\begin{equation}
  P_{X_1 \cdots X_N} = \prod_i P_{X_i|X_{\Pi_i}},
\end{equation}
where $X_{\Pi_i}$ is the set containing the parents of $X_i$ (the vertices from which an edge points to $i$).

\noindent
The following \textbf{notations} will hold for the rest of this chapter: 
\begin{itemize}
 \item random variables are capitalized: $X_i$,
 \item realizations of random variables are lower case: $x_i$,
 \item a set of random variables $\{X_1 \cdots X_N \}$ is noted $\cv{X}$,
 \item a set of realizations of $\cv{X}$ is noted $\cv{x}$,
 \item the subset of random variables of indices in $S$ is noted $X_S$.
\end{itemize}

\subsubsection{Undirected GMs}
\index{Undirected}
Another type of graphical model is the undirected graphical model. In that case, we define the graphical model not through the \textbf{factorization}, but through \textbf{independence}.\index{Graphical model}

Let:
\begin{align}
 &\mathcal{G}(\mathcal{V},\mathcal{E})   &\text{be an undirected graphical model, where}  \notag \\
 &\mathcal{V} = \{ 1 ,\cdots,N \} \quad &\text{is the set of vertices, and}  \notag \\  
 &\mathcal{E} \subseteq \mathcal{V} \times \mathcal{V} \quad &\text{is the set of edges.}\nonumber
\end{align}
Each vertex $i \in \mathcal{V}$ of this GM represents one random variables $X_i$, and each edge $(i,j) \in \mathcal{E}$ represents a conditional dependence.
As the GM is undirected, we have $(i,j) \equiv (j,i)$.

We define:
\begin{equation}
 N(i) \triangleq \{ j \in \mathcal{V} | (i,j)\in \mathcal{E} \} \quad \text{the set containing the neighbours of $i$.}
\end{equation}

\index{Undirected}
Undirected graphical model captures following dependence:\index{Graphical model}
\begin{equation}
 P_{X_i| X_{\mathcal{V} \backslash \{ i \} }} \equiv P_{X_i | X_{N(i)}},
\end{equation}
meaning that only variables connected by edges have a conditional dependence.

Let $A \subset \mathcal{V}$, $B \subset \mathcal{V}$, $C \subset \mathcal{V}$. 
We write that $X_A \perp X_B \mid X_C$ if $A$ and $B$ are disjoint and if all pathes leading from one element of $A$ to one element of $B$ 
lead over an element of $C$, as is illustrated in Fig. \ref{fig:3}. In other words, if we remove $C$, then $A$ and $B$ are unconnected (Fig. \ref{fig:4}).
\begin{figure}[h]
\begin{center}
 \includegraphics[width=0.4\textwidth]{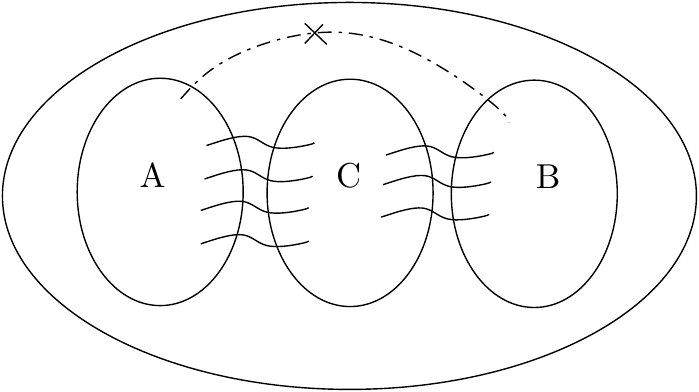}
\end{center}
\caption{Schematic view of a graphical model in which $X_A \perp X_B \mid X_C$. All paths leading from $A$ to $B$ go through $C$.}
 \label{fig:3}
\end{figure}
\begin{figure}[h]
\begin{center}
 \includegraphics[width=1\textwidth]{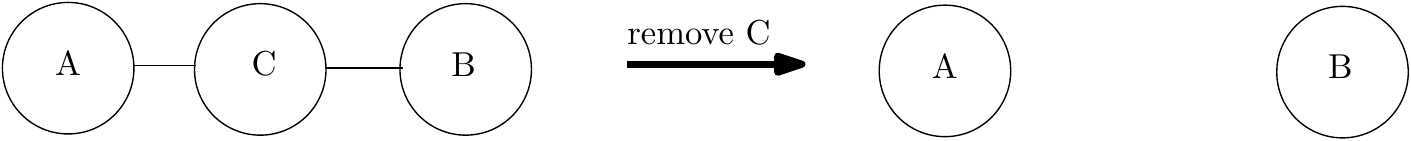}
\end{center}
\caption{Simple view showing the independence of $A$ and $B$ conditioned on $C$.}
 \label{fig:4}
\end{figure}

Undirected GMs are also called \textbf{Markov random fields} (MRF).
\index{Undirected}

\subsubsection{Cliques}(\textbf{Definition})
A clique is a subgraph of a graph in which all possible pairs of vertices are linked by an edge. A maximal clique
is a clique that is contained by no other clique.
\begin{figure}[h]
\begin{center}
 \includegraphics[width=0.4\textwidth]{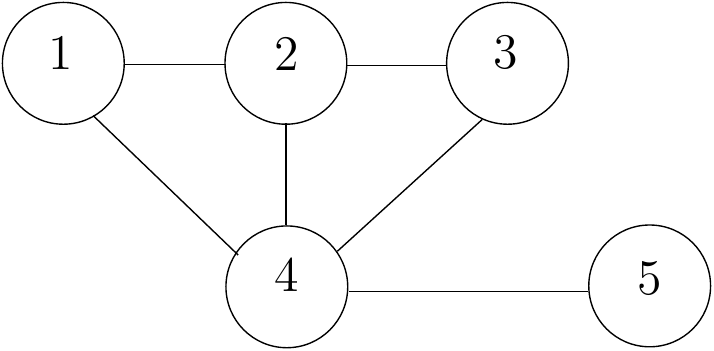}
\end{center}
\caption{In this graphical model, the maximal cliques are $\{1,2,4\}$, $\{2,3,4\}$ and $\{4,5\}$. }
 \label{fig:5}
\end{figure}

\begin{theorem} (\cite{HC1971})
Given a MRF $\mathcal{G}$ and a probability distribution $P_{\cv{X}}(\cv{x}) > 0$. Then:
\begin{equation}
 P_{\cv{X}} ( \cv{x} ) \propto \prod_{C \in \mathcal{C}} \phi_C(x_C)
\end{equation}
where $\mathcal{C}$ is the set of cliques of $\mathcal{G}$.
\end{theorem}


\begin{proof}(\cite{Gri1973}) for $\mathcal{X} = \{ {\rm 0} , {\rm 1} \}$.\\
We will show the following, equivalent formulation:
\begin{equation}
  P_{\cv{X}} ( \cv{x} ) \propto e^{\sum_{C \in \mathcal{C}} V_C(x_C) }
\end{equation}
by exhibiting the solution:
\begin{equation}
 V_C(x_C) =
 \begin{cases}
  Q(C) & \text{if } x_C = \mathbb{1}_C , \\ 
  0 & \text{otherwise,}
 \end{cases}
\end{equation}
with
\begin{align}
 Q(C) &= \sum_{A \subseteq C} (-1)^{|C-A|} \underbrace{\ln P_{\cv{X}}\left( x_A = \mathbb{1}_A , x_{V \backslash A} = \mathbb{0} \right)}_{\triangleq G(A)}.
\end{align}

Suppose we have an assignement $\cv{X} \mapsto N(\cv{X}) = \{ i | x_i=1 \}$. We want to prove that:
\begin{align}
 G(N(\cv{X})) &\triangleq \ln P_{\cv{X}}(\cv{x}),  \notag \\
	      &= \sum_{C \in \mathcal{C}} V_C(x_C),  \notag \\
	      &= \sum_{C \subseteq N(\cv{x})} Q(C).
\end{align}

This is equivalent to proving the two claims:
\begin{align}
 &{\rm C1:} \quad \forall S \subset \mathcal{C}, \quad G(S) = \sum_{A \subseteq S} Q(A)  \notag \\
 &{\rm C2:} \quad \text{if $A$ is not a clique, } \quad Q(A) = 0 \notag
\end{align}

Let us begin by proving C1:
\begin{align}
 \sum_{A \subseteq S} Q(A) &= \sum_{A \subseteq S} \sum_{B \subseteq A} (-1)^{|A-B|} G(B) \notag \\
 &= \sum_{B\subseteq S} G(B) \left( \sum_{B \subseteq A \subseteq S} (-1)^{|A-B|} \right) 
\end{align}
where we note that the term in brackets is zero except when $B=S$, because we can rewrite it as
\begin{align}
 \sum_{0 \leq l \leq k} (-1)^l \binom{l}{k} = (-1 + 1)^k = 0.
\end{align}
Therefore $\quad G(S) = \sum_{A \subseteq S} Q(A)$.
\vspace{0.5cm}

For C2, suppose that $A$ is not a clique, which allows us to choose  $(i,j) \in A$ with $(i,j) \notin \mathcal{E}$. Then
\begin{align} 
 Q(A) &= \sum_{B \subseteq A \backslash \{ i,j \} } (-1)^{|A-B|} \left[ G(B) - G(B+i) + G(B+i+j) - G(B+j) \right]. \notag
\end{align}
Let us show that the term in brackets is zero by showing
\begin{align}
 G(B+i+j)-G(B+j) &= G(B+i) - G(B)  \notag
\end{align}
or equivalently
\begin{align}
 \textstyle \ln \frac{P_X \left( x_B=\mathbb{1}_B , x_i=1, x_j=1, x_{\mathcal{V} \backslash \{ i,j,B \}} = \mathbb{0} \right) }{P_X \left( x_B=\mathbb{1}_B , x_i=0, x_j=1, x_{\mathcal{V}\backslash \{ i,j,B\}} = \mathbb{0} \right) } &= \textstyle \ln \frac{P_X \left( x_B=\mathbb{1}_B, x_i=1, x_j=0, x_{\mathcal{V}\backslash \{ i,j,B\} } = \mathbb{0} \right)}{P_X \left( x_B=\mathbb{1}_B , x_i=0, x_j=0, x_{\mathcal{V} \backslash \{ i,j,B\}} = \mathbb{0} \right)} \notag, 
\end{align}
where $\mathcal{V}\backslash \{ i,j,B \}$ stands for the set of all vertices except $i$, $j$ and those in $B$.\\
We see that the only difference between the left-hand side and the right-hand side is the value taken by $x_j$. Using Bayes' rule, we can rewrite both the right-hand side and the left-hand side under the form
\begin{align}
 \textstyle \ln \frac{P_X \left( X_i = 1 | X_j = \pm 1, X_B=\mathbb{1}_B, X_{\mathcal{V} \backslash \{ i,j,B \}} = \mathbb{0} \right) }{P_X \left(X_i = 0 | X_j = \pm 1, X_B=\mathbb{1}_B, X_{\mathcal{V} \backslash \{ i,j,B \}} = \mathbb{0} \right)) }. \notag
\end{align}
As $(i,j) \notin \mathcal{E}$, the conditional probabilities on $X_i$ do not depend on the value taken by $X_j$,  
and therefore the right-hand side equals the left-hand side, $Q(A) = 0$ and C2 is proved.
\end{proof}

\subsection{Factor graphs}
Thanks to the Hammersley-Clifford theorem, we know that we can write a probability distribution corresponding to
a MRF $\mathcal{G}$ in the following way
\begin{align}
 P_{\cv{X}}(\cv{x} ) \propto \prod_{C \in \mathcal{C}^*} \phi_C (x_C)\label{eq:fg}
\end{align}
where $\mathcal{C}^*$ is the set of maximal cliques of $G$.
In a general definition, we can also write
\begin{align}
  P_{\cv{X}}(\cv{x} ) \propto \prod_{F \in \mathcal{F}} \phi_F (x_F)
\end{align}
where the partition $\mathcal{F} \subseteq 2^\mathcal{V}$ has nothing to do with any underlying graph.

In what follows, we give two examples in which introducing factor graphs is a natural approach to an inference problem. \index{Inference}

\subsubsection{Image processing}
We consider an image with binary pixels ($\mathcal{X} = \{-\text{1},\text{1} \}$), and a probability distribution:
\begin{equation}
 p(\cv{x})  \propto e^{\sum_{i \in V} \theta_i x_i + \sum_{(i,j) \in E} \theta_{ij} x_i x_j}
\end{equation}
\begin{figure}[h]
 \begin{center}
  \includegraphics[width=0.6\textwidth]{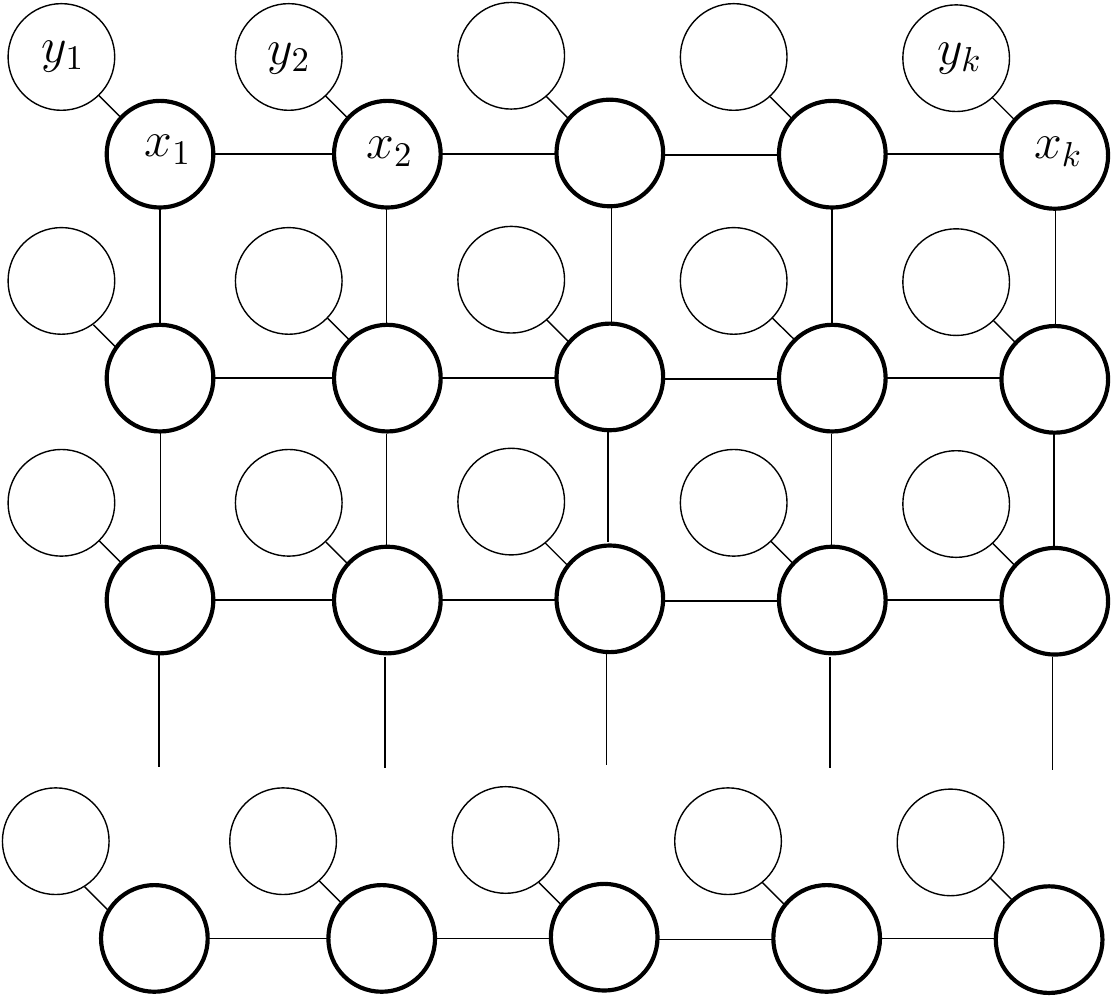}
 \end{center}
 \caption{Graphical model representing a 2D image. The fat circles correspond to the pixels of the image $x_k$, and each one is 
 linked to a noisy measurement $y_k$. Adjacent pixels are linked by edges that allow modelling the assumed smoothness of the image.}
\label{fig:6}
\end{figure}

For each pixel $x_k$, we record a noisy version $y_k$. We consider natural images, in which big jumps in intensity between 
two neighbouring pixels are unlikely. This can be modelled with:
\begin{align}
 &a \sum_{i} x_i y_i + b \sum_{(i,j) \in \mathcal{E}} x_i x_j
\end{align}

This way, the first term pushes $x_k$ to match the measured value $y_k$, while the second term favours piecewise constant images. 
We can identify $\theta_i \equiv a y_i$ and $ \theta_{ij} \equiv b$.

\subsubsection{Crowd-sourcing}
Crowd-sourcing is used for tasks that are easy for humans but difficult for machines, and that are as hard to verify as to evaluate.
Crowd-sourcing then consists in assigning to each of $M$ human ``workers" a subset of $N$ tasks to evaluate, and to collect their answers $A$.
Each worker has a different error probability $p_i \in \{ \frac{1}{2} , 1 \}$: either he gives random answers, or he is fully reliable . The goal is to 
infer both the correct values of each task, $t_j$, and the $p_i$ of each worker. The factor graph corresponding to that problem is represented in Fig \ref{fig:7}.
\begin{figure}[h]
 \begin{center}
  \includegraphics[width=0.5\textwidth]{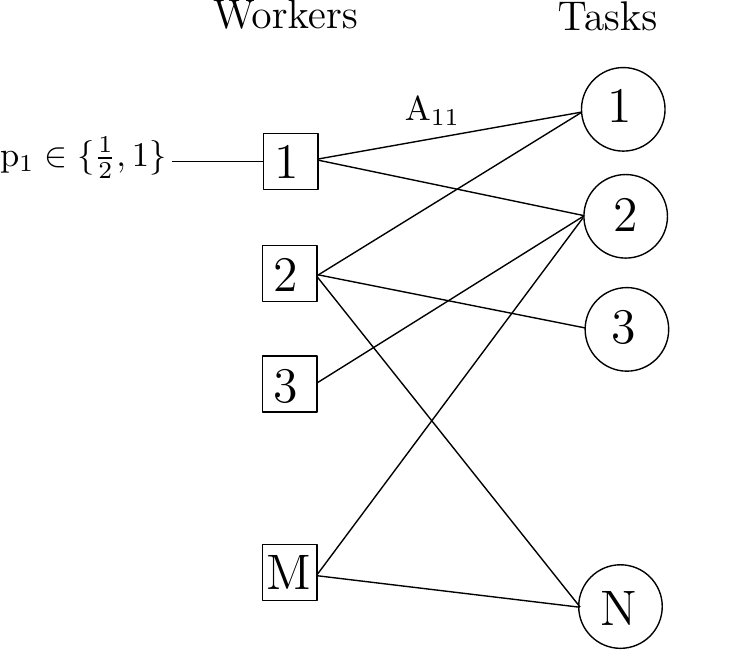}
 \end{center}
 \caption{Graphical model illustrating crowd-sourcing. Each worker $i$ is assigned a subset of the tasks for evaluation, and for each
 of those tasks $a$, his answer ${\rm A}_{ia}$ is collected.}
\label{fig:7}
\end{figure}

The conditional probability distribution of $\cv{t}$ and $\cv{p}$ knowing the answers $A$ reads
\begin{align}
 P_{\cv{t},\cv{p} | \cv{A} } &\propto P_{\cv{A} | \cv{t},\cv{p} } P_{\cv{t}, \cv{p}}  \notag \\
 &\propto P_{\cv{A} | \cv{t},\cv{p} }
\end{align}
where we assumed a uniform distribution on the joint probability $P_{\cv{t}, \cv{p}}$.
Then
\begin{align}
 P_{\cv{A} | \cv{t}, \cv{p}} &= \prod_e P_{A_e | t_e, p_e}
\end{align}
with
\begin{align}
 P_{A_e | t_e, p_e} = \left( \left( \frac{p_e}{1-p_e} \right)^{A_e t_e} (1-p_e) p_e \right)^{\frac{1}{2}}.
\end{align}

\index{MAP}\index{MARG}
\subsection{MAP and MARG}
\emph{MAP}.
The MAP problem consists in solving:
\begin{align}
 {\rm max}_{\cv{x} \in \{ 0, 1 \}^N} \sum_i \theta_i x_i+ \sum_{(i,j) \in \mathcal{E}} \theta_{ij} x_i x_j.
\end{align}
When $\theta_{ij} \rightarrow - \infty$, neighbouring nodes can not be in the same state anymore. 
This is the hard-core model, which is very hard to solve.\\

\noindent
\emph{MARG}.
The MARG focuses on the evaluation of marginal probabilities, depending on only one random variable, for instance:
\begin{equation}
 P_{X_1} (0) = \frac{Z(X_1 = 0)}{Z}
\end{equation}
as well as conditional marginal probabilities:
\begin{align}
 P_{X_2 | X_1} ( X_2=0 | X_1=0) &= \frac{Z(X_1=0,X_2=0)}{Z(X_1=0)} \\
 P_{X_N | X_1 \cdots X_{N(1)} } ( X_N=0 | X_1 \cdots X_{N-1} =0) &= \frac{Z(\text{all } 0)}{Z(\text{all but } X_N \text{ are } 0)}
\end{align}
\begin{align}
 P_{X_1}(0) \times \cdots \times P_{X_N | X_1 \cdots X_{N-1}}(0) = \frac{1}{Z}
\end{align}

Both of these problems are computationally hard. 
Can we design efficient algorithms to solve them?

\section{Inference Algorithms: Elimination, Junction Tree and Belief Propagation}
\index{Junction tree}
\index{Inference}
\index{Belief propagation}

\index{Hardness}
\index{MAP}\index{MARG}
In the MAP and MARG problems described previously, the hardness comes from the fact that with growing instance size, 
the number of combinations of variables over which to maximize or marginalize becomes quickly intractable.
But when dealing with GMs, one can exploit the structure of the GM in order to reduce the number of combinations that have to
be taken into account. Intuitively, the smaller the connectivity of the variables in the GM is, the smaller this number of 
combination becomes. We will formalize this by introducing the elimination algorithm, that gives us a systematic way of 
making fewer maximizations/marginalizations on a given graph. We will see how substantially 
the number of operations is reduced on a graph that is not completely connected.

\subsection{The elimination algorithm}

We consider the GM in Fig. \ref{fig:graph1} which is not fully connected. The colored subgraphs represent the maximal cliques.
\begin{figure}[!h]
\begin{center}
	\includegraphics[width=0.30\columnwidth]{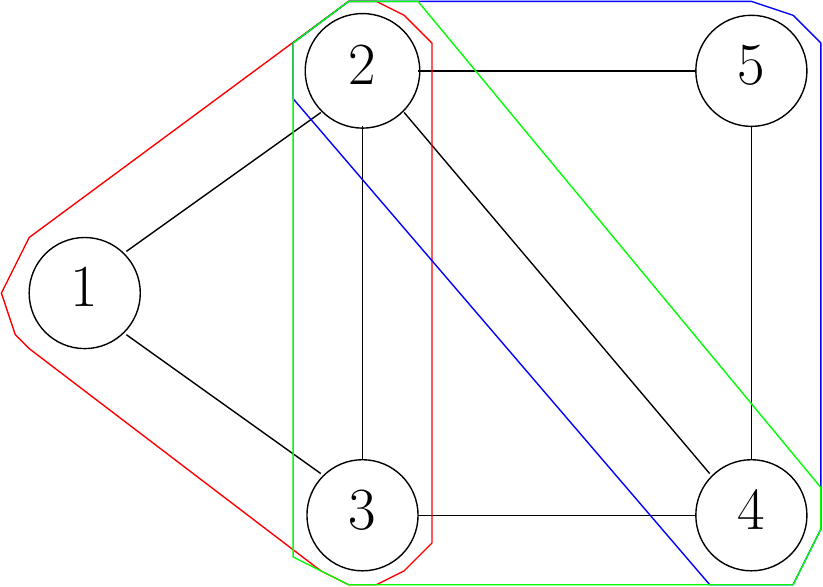}
\end{center}
\caption{A GM and its maximal cliques.}
\label{fig:graph1}
\end{figure}

\noindent
Using decomposition \eqref{eq:fg}, we can write
\begin{align}
P_{\cv{X}} ( \cv{x})\propto \phi_{123}(x_1,x_2,x_3).\;\phi_{234}(x_2,x_3,x_4).\;\phi_{245}(x_2,x_4,x_5).
\label{factorized}
\end{align}
We want to solve the MARG problem on this GM, for example 
for calculating the marginal probability of $x_1$:
\begin{align}
 P_{X_1}(x_1) &= \sum_{x_2,x_3,x_4,x_5} P_{\cv{X}}(\cv{x}).
\end{align}
A priori, this requires to evaluate $|\XX|^4$ terms, each of them taking $|\XX|$ different values. In the end, 
$3 |\XX| |\XX|^4$ operations are needed for calculating this marginal naively. But if we take advantage of 
the factorized form (\ref{factorized}), we can eliminate some of the variables. The elimination process goes along 
these lines:
\begin{align}
P_{X_1} (x_1)
&\propto \sum_{x_2,x_3,x_4,x_5} \phi_{123}(x_1,x_2,x_3).\;\phi_{234}(x_2,x_3,x_4).\;\phi_{245}(x_2,x_4,x_5),\label{eq:beginfer}\\
&\propto \sum_{x_2,x_3,x_4} \phi_{123}(x_1,x_2,x_3).\;\phi_{234}(x_2,x_3,x_4).\;\sum_{x_5}\phi_{245}(x_2,x_4,x_5),\\
&\propto \sum_{x_2,x_3,x_4} \phi_{123}(x_1,x_2,x_3).\;\phi_{234}(x_2,x_3,x_4).\;m_5(x_2,x_4),\\
&\propto \sum_{x_2,x_3} \phi_{123}(x_1,x_2,x_3).\;\sum_{x_4}\phi_{234}(x_2,x_3,x_4).\;m_5(x_2,x_4),\\
&\propto \sum_{x_2,x_3} \phi_{123}(x_1,x_2,x_3).\;m_4(x_2,x_3),\\
&\propto \sum_{x_2}\left (\sum_{x_3} \phi_{123}(x_1,x_2,x_3)\;m_4(x_2,x_3)\right),\\
&\propto \sum_{x_2}m_3(x_1,x_2),\\
&\propto m_2(x_1).\label{eq:endinfer}
\end{align}
With this elimination process made, the number of operations necessary to compute the marginal scales as $|\XX|^3$ instead of $|\XX|^5$,
thereby greatly reducing the complexity of the problem by using the structure of the GM.
Similarly, we can rewrite the MAP problem as follows\index{MAP}
\begin{align}
&\max_{x_1,x_2,x_3,x_4,x_5}\phi_{123}(x_1,x_2,x_3).\;\phi_{234}(x_2,x_3,x_4).\;\phi_{245}(x_2,x_4,x_5),\\
&= \max_{x_1,x_2,x_3,x_4} \phi_{123}(x_1,x_2,x_3).\;\phi_{234}(x_2,x_3,x_4).\;\max_{x_5}\phi_{245}(x_2,x_4,x_5),\\
&= \max_{x_1,x_2,x_3,x_4} \phi_{123}(x_1,x_2,x_3).\;\phi_{234}(x_2,x_3,x_4).\;m_5^\star(x_2,x_4),\\
&= \max_{x_1,x_2,x_3} \phi_{123}(x_1,x_2,x_3).\;\max_{x_4}\phi_{234}(x_2,x_3,x_4).\;m_5^\star(x_2,x_4),\\
&= \max_{x_1,x_2,x_3} \phi_{123}(x_1,x_2,x_3).\;m_4^\star(x_2,x_3),\\
&= \max_{x_1,x_2}\left (\max_{x_3} \phi_{123}(x_1,x_2,x_3)\;m_4^\star(x_2,x_3)\right),\\
&= \max_{x_1,x_2}m_3^\star(x_1,x_2),\\
&= \max_{x_1}\left (\max_{x_2}m_3^\star(x_1,x_2)\right),
\end{align}
leading to 
\begin{align}
x_1^\star=\argmax_{x_1} m_2^\star(x_1).
\end{align}

\index{MARG}
Just like for the MARG problem, the complexity is reduced from $|\XX|^5$ (a priori) to $|\XX|^3$.
We would like to further reduce the complexity of the marginalizations (in $|\mathcal{X}|^3$). 
One simple idea would be to reduce the GM into a linear graph as in Fig. \ref{linear_graph}.
\begin{figure}[!h]
\begin{center}
	\includegraphics[width=0.5\columnwidth]{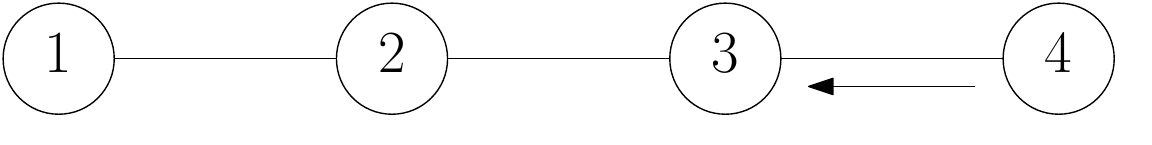}
\end{center}
\caption{A linear graph. Each marginalization is computed in $|\mathcal{X}|^2$ operations. }
\label{linear_graph}
\end{figure}\\

\begin{figure}[!h]
\begin{center}
	\includegraphics[width=0.5\columnwidth]{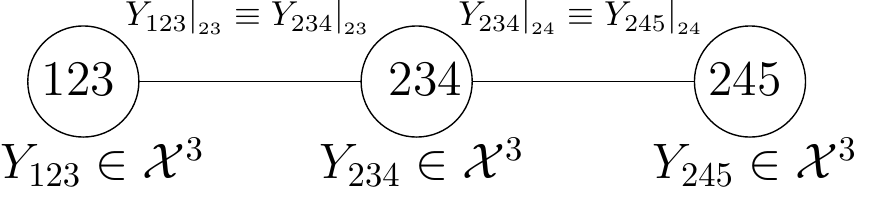}
\end{center}
\caption{Linear GM obtained by grouping variables.}
\label{groupedVars}
\end{figure}
By grouping variables in the GM (Fig. \ref{fig:graph1}), it is in fact possible
to obtain a linear graph, as shown in Fig. \ref{groupedVars},
with the associated potentials $\phi_{123}(Y_{123})$, $\phi_{234}(Y_{234})$ and $\phi_{245}(Y_{245})$ 
and the consistency constraints $Y_{123}|_{_{23}}\equiv Y_{234}|_{_{23}}$ and $Y_{234}|_{_{24}}\equiv Y_{245}|_{_{24}}$.
For other GMs, the simplest graph achievable by grouping variables might be a tree instead of a simple chain.
But not all groupings of variables will lead to a tree graph that correctly represents the problem.
In order for the grouping of variables to be correct, we need to build the tree attached to the maximal cliques, and
we have to resort to the Junction Tree property.\index{Junction tree}

\subsection{Junction Tree property and chordal graphs}
The Junction Tree property allows us to find groupings of variables under which the GM becomes a tree (if such groupings 
exist). On this tree, the elimination algorithm will need a lower number of maximizations/marginalizations than on the initial GM.
However, there is a remaining problem: running the algorithm on the junction tree does not give a straightforward solution to the 
initial problem, as the variables on the junction tree are groupings of variables of the original problem. This means that 
further maximizations/marginalizations are then required to have a solution in terms of the variables of the initial problem.
\subsubsection{Junction Tree (JCT) property}(\textbf{Definition})
A graph $\mathcal{G}=(\mathcal{V},\mathcal{E})$ is said to possess the JCT property if it has a Junction Tree $\mathcal{T}$ which is defined as follows: it is a tree graph such that 
\begin{itemize}
 \item its nodes are maximal cliques of $\mathcal{G}$
 \item an edge between nodes of $\mathcal{T}$ is allowed only if the corresponding cliques share at least one vertex
 \item for any vertex $v$ of $\mathcal{G}$, let $\mathcal{C}_v$ denote set of all cliques containing $v$. Then $\mathcal{C}_v$ forms a connected sub-tree of $\mathcal{T}$. 
\end{itemize}
\noindent 
Two questions then arise
\begin{itemize}
\item Do all graphs have a JCT?
\item If a graph has a JCT, how can we find it?\\
\end{itemize}

\subsubsection{Chordal graph}(\textbf{Definition})
A graph is chordal if all of its loops have chords. Fig. \ref{chordalGraphs} gives an illustration of the concept.
\begin{figure}[!h]
\begin{center}
	\includegraphics[width=0.25\columnwidth]{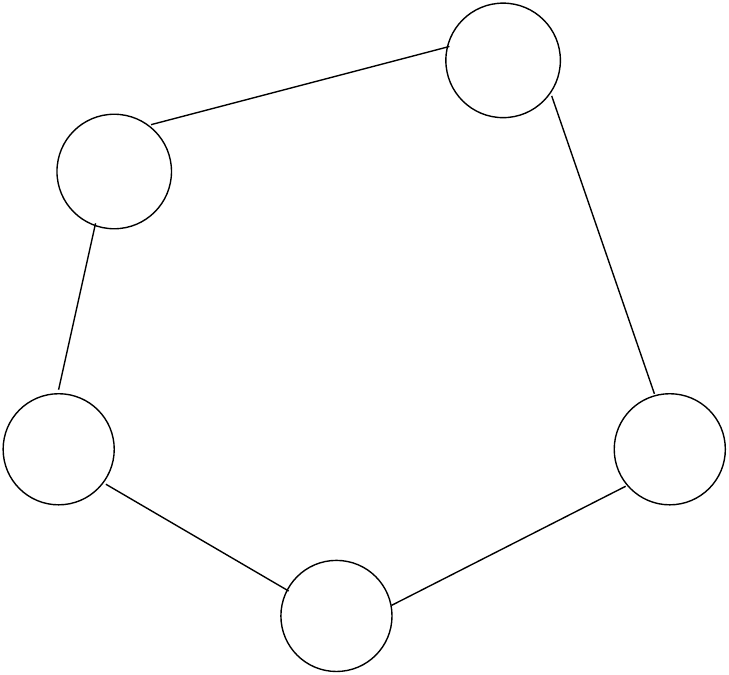}\hspace{0.5cm}
	\includegraphics[width=0.25\columnwidth]{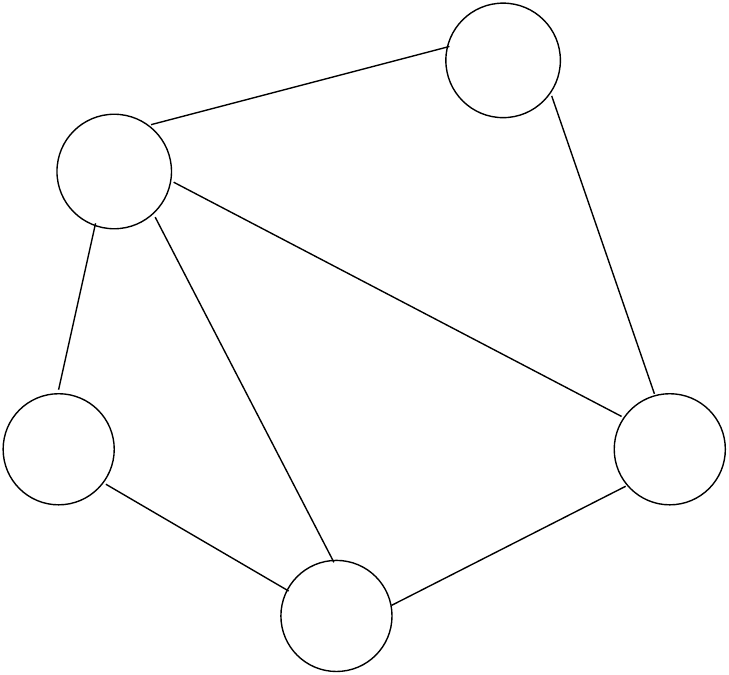}	
\end{center}
\caption{The graph on the left is not chordal, the one on the right is.}
\label{chordalGraphs}
\end{figure}

\begin{proposition}
$\mathcal{G}$ has a junction tree $\Leftrightarrow $ $\mathcal{G}$ is a chordal graph.\index{Junction tree}
\end{proposition}

\begin{proof} of the implication $\Leftarrow$.
Let us take a chordal graph $\mathcal{G} = (\mathcal{V},\mathcal{E})$ that is not complete, as represented in Fig. \ref{fig:graph4}.
\begin{figure}[!h]
\begin{center}
	\includegraphics[width=0.5\columnwidth]{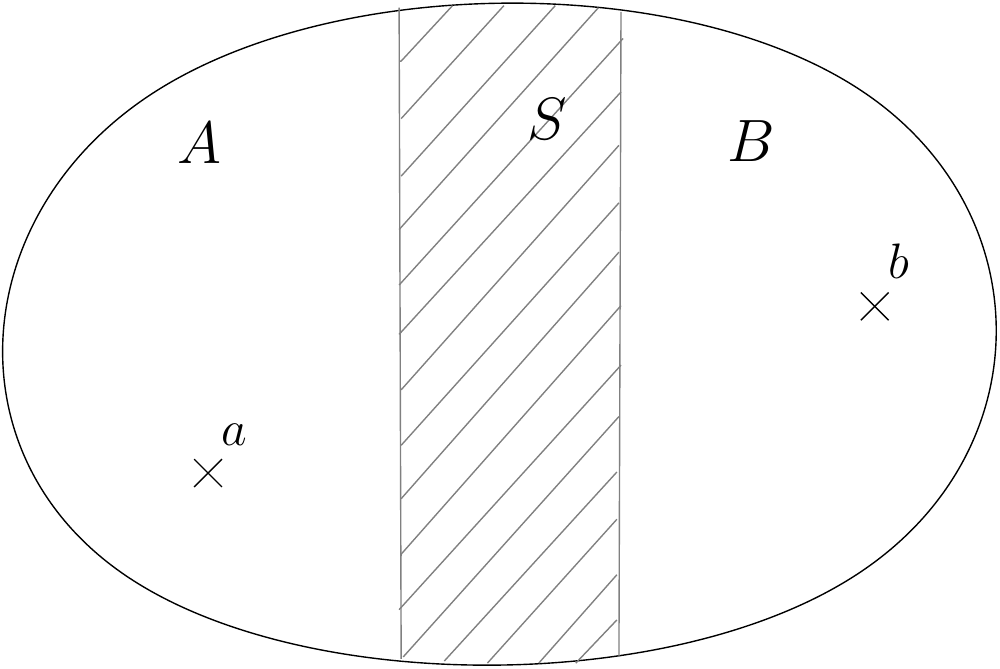}
\end{center}
\caption{On a chordal graph that is not complete, two vertices $a$ and $b$ that are not connected, separated by a subgraph $S$ that is fully connected. }
\label{fig:graph4}
\end{figure}\\

We will use the two following lemmas that can be shown to be true:
\begin{enumerate}
 \item If $\mathcal{G}$ is chordal, has at least three nodes and is not fully connected, then $\mathcal{V}=\mathcal{A} \cup \mathcal{B} \cup \mathcal{S}$, where all three sets are disjoint and $\mathcal{S}$ is a fully 
 connected subgraph that separates $\mathcal{A}$ from $\mathcal{B}$.
 \item If $\mathcal{G}$ is chordal and has at least two nodes, then $\mathcal{G}$ 
 has at least two nodes each with all neighbors connected. 
 Furthermore, if $\mathcal{G}$ is not fully connected, then there exist two nonadjacent nodes each with all its neighbors connected.
\end{enumerate}

The property \emph{``If $\mathcal{G}$ is a chordal graph with $N$ vertices , then it has a junction tree.''} can be shown by induction on $N$.
For $N=2$, the property is trivial. Now, suppose that the property is true for all integers up to $N$. 
Consider a chordal graph with $N+1$ nodes. By the second lemma, $\mathcal{G}$ has a node $a$ with all its neighbors connected. 
Removing it creates a graph $\mathcal{G}'$ which is chordal, and therefore has a JCT, $\mathcal{T}'$. Let $C$ be the maximal clique
that $a$ participates in. Either $C \setminus a$ is a maximal-clique node in $\mathcal{T}'$, and in this case adding $a$ to this clique node 
results in a junction tree $\mathcal{T}$ for $\mathcal{G}$. Or $C \setminus a$ is not a maximal-clique node in $\mathcal{T}'$. Then,
$C \setminus a$ must be a subset of a maximal-clique node $D$ in $\mathcal{T}'$. Then, 
we add $C$ as a new maximal-clique node in $\mathcal{T}'$, which we connect to $D$ to obtain a junction tree $\mathcal{T}$ for 
$\mathcal{G}$.\index{Junction tree}
\end{proof}


\subsubsection{Procedure to find a JCT}
Let $G$ be the initial GM, and $\mathcal{G}(\mathcal{V},\mathcal{E})$ be the GM in which $\mathcal{V}$ is the set of maximal cliques of $G$ and 
 $(c_1, c_2) \in \mathcal{E}$ if the maximal cliques $c_1$ and $c_2$ share a vertex.
 Let us take $e=(c_1,c_2)$ with $c_1, c_2\in\mathcal{V}$ and  define the weight of $e$ as $w_e=|c_1\cap c_2|$.
 Then, finding a junction tree of $G$ is equivalent to finding the max-cut spanning tree of $\mathcal{G}$.
Denoting by $T$ the set of edges in a tree, we define the weight of the tree as
\begin{align}
W(T)
&=\sum_{e\in T}w_e \\
&=\sum_{e\in T}|c_1\cap c_2| \notag \\
&=\sum_{v\in V}\sum_{e\in T} \mathbb{1}_{\lbrace v\in e\rbrace} \notag.
\end{align}
and we claim that $W(T)$ is maximal when T is a JCT.\\

\textbf{Procedure} to get the maximum weighted spanning tree
\begin{itemize}
\item List all edges in a decreasing order,
\item Include $e_i$ in $\mathcal{E}_{i-1}$ if you can.
\end{itemize}
what we are left with at the end of the algorithm is the maximal weight spanning tree.

\subsubsection{Tree width}(\textbf{Definition})
The width of a tree decomposition is the size of its maximal clique minus one.\\~\\
\newpage
\emph{Toy examples}
\begin{figure}[!h]
\begin{center}
	\includegraphics[width=0.25\columnwidth]{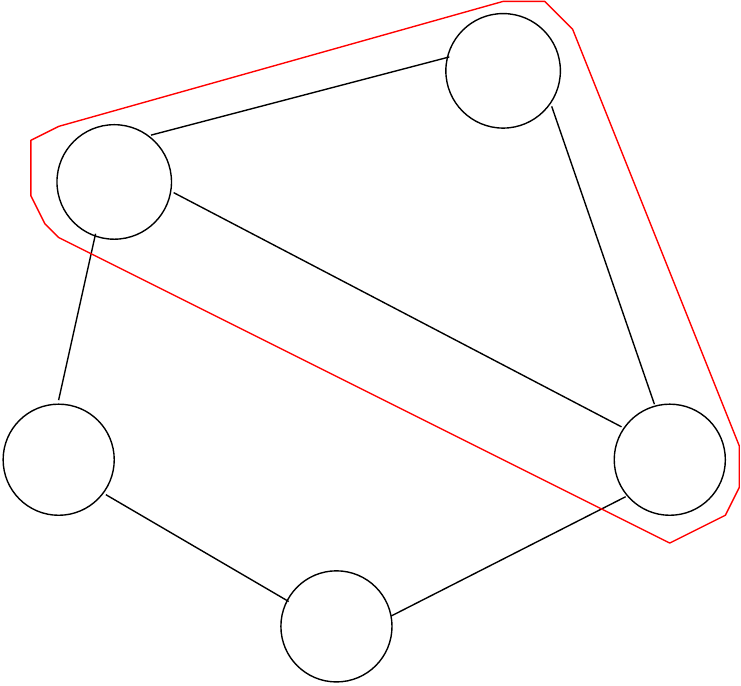}\hspace{0.5cm}
	\includegraphics[width=0.4\columnwidth]{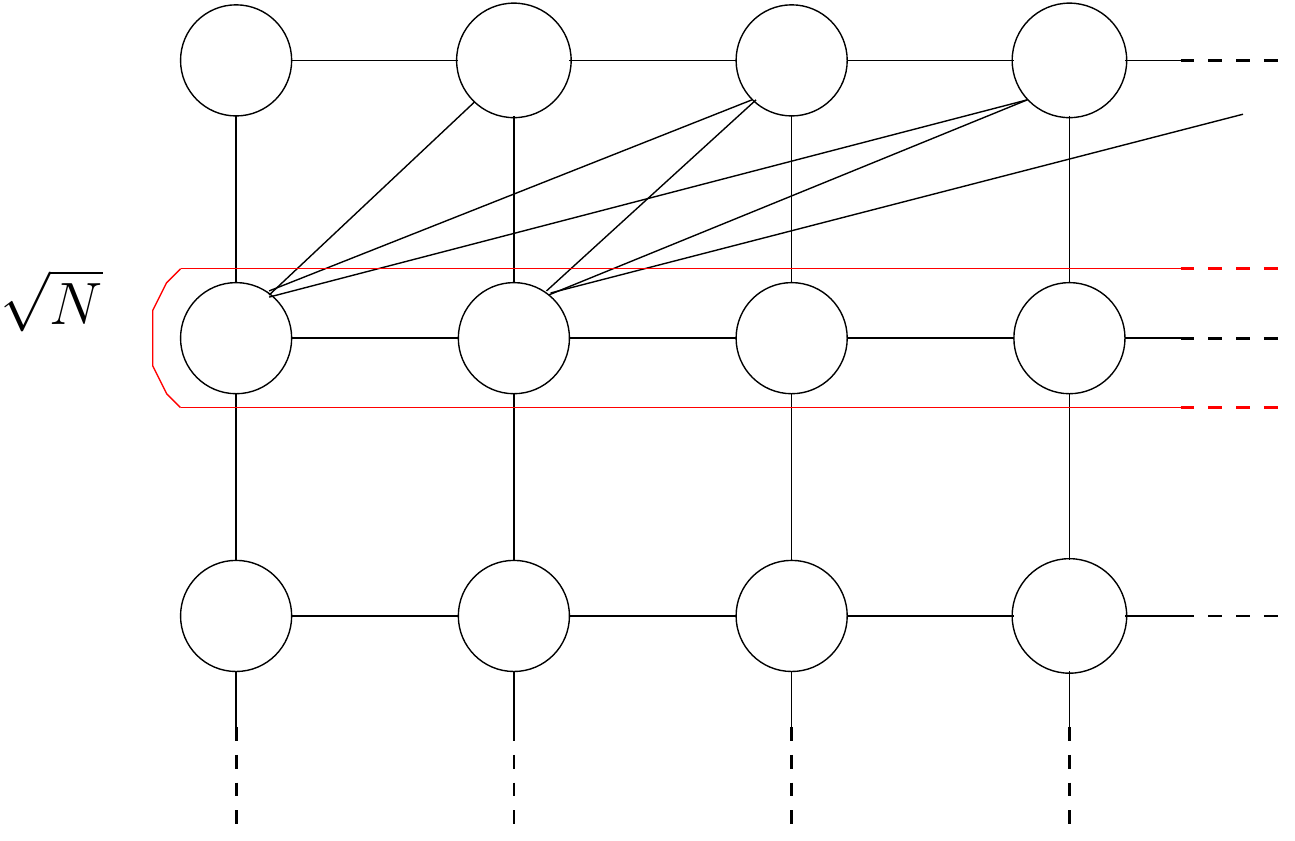}
	\caption{tree width = 2 (left), tree width = $\sqrt{N}$ (right)}	
\end{center}
\end{figure}

\subsection{Belief propagation (BP) algorithms}
\index{Belief propagation}

Until now, everything we have done was exact. The elimination algorithm is an exact algorithm. But as we are interested in \textbf{efficient algorithms},
as opposed to exact algorithms with too high complexities to actually end in reasonable time, we will from now on \textbf{introduce approximations}.

\begin{figure}[!h]
\begin{center}
	\includegraphics[width=0.3\columnwidth]{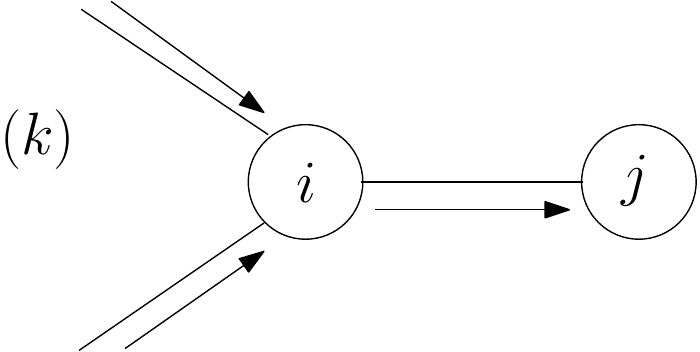}
\end{center}
\caption{Message passing on a graph.}
\end{figure}
Coming back to the elimination algorithm \eqref{eq:beginfer}-\eqref{eq:endinfer}, we can generalize the notations used as
\begin{align}
m_i(x_j)\propto \sum_{x_i} \phi_i(x_i) .\;\phi_{i,j}(x_i,x_j).\prod_k m_k(x_i).
\end{align}
Considering now the same but oriented GM (arrows on figure above), we get
\begin{align}
m_{i\rightarrow j}(x_j)\propto \sum_{x_i} \phi_i(x_i). \;\phi_{i,j}(x_i,x_j).\prod_{k\in N(i)\setminus j} m_{k\rightarrow i}(x_i), 
\end{align}
where $N(i)$ is the neighbourhood of $x_i$.\\

\newpage
\index{MARG}
The MARG problem can then be solved using the \textbf{sum-product} procedure.
\begin{algorithm}
\textbf{Sum-product BP}
\begin{itemize}
\item $t=0$,
\begin{align}
\forall (i,j)\in E, \forall (x_i,x_j) \in \XX^2:  \quad m^0_{i\rightarrow j}(x_j)=m^0_{j\rightarrow i}(x_i)=1.
\end{align}
\item $t>0$,
\begin{align}
m^{t+1}_{i\rightarrow j}(x_j)&\propto\sum_{x_i}\phi_i(x_i).\;\phi_{ij}(x_i,x_j).\prod_{k\in N(i)\setminus j} m^t_{k\rightarrow i}(x_i), \\
P^{t+1}_{X_i}(x_i)&=\prod_{k\in N(i)}m^{t+1}_{k\rightarrow i}(x_i).
\end{align}
\end{itemize}
\end{algorithm}

\index{MAP}
While, for the MAP problem, the \textbf{max-sum} procedure is considered.
\begin{algorithm}
\textbf{Max-sum BP}
\begin{itemize}
\item $t=0$,
\begin{align}
m^0_{i\rightarrow j}(x_j)=m^0_{j\rightarrow i}(x_i)=1.
\end{align}
\item $t>0$,
\begin{align}
m^{t+1}_{i\rightarrow j}(x_j)&\propto\max_{x_i}\phi_i(x_i).\;\phi_{ij}(x_i,x_j).\prod_{k\in N(i)\setminus j} m^t_{k\rightarrow i}(x_i), \\
x_i^{t+1}&=\argmax_{x_i}\phi_i(x_i).\;\prod_{k\in N(i)}m^{t+1}_{k\rightarrow i}(x_i).
\end{align}
\end{itemize}
\end{algorithm}

Note: here, we use only the potentials of pairs. But in case of cliques, we have to consider the JCT and iterate on it. 
To understand this point, let us apply the sum-product algorithm on factor graphs.
\subsubsection{Factor graphs}
\begin{figure}[!h]
\begin{center}
	\includegraphics[width=0.3\columnwidth]{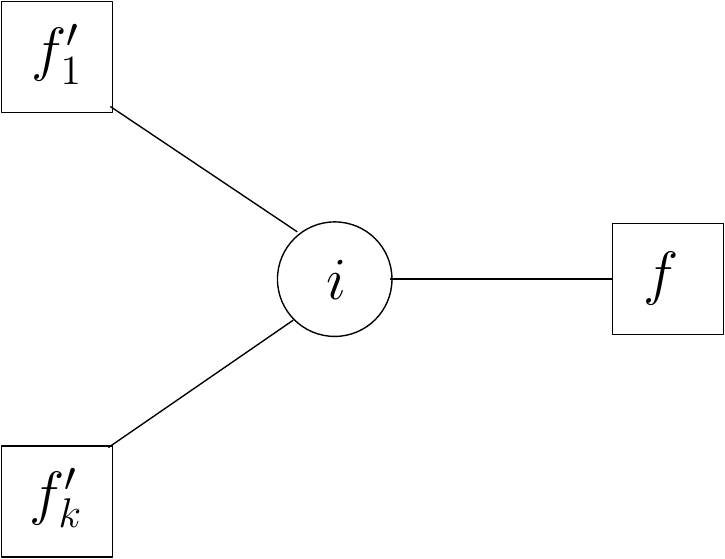}
\end{center}
\caption{A simple factor graph.}
\label{fig:graph10}
\end{figure}
Considering the general notations in Fig. \ref{fig:graph10}, 
the sum-product BP algorithm is particularized such that
\begin{align}
m^{t+1}_{i\rightarrow f}(x_i)&=\prod_{f'\in N(i)\setminus f}m^{t}_{f'\rightarrow i}(x_i), \\
m^{t+1}_{f\rightarrow i}(x_i)&=\sum_{x_j, j\in N(f)\setminus i}f(x_i,x_j) \prod_{j\in N(f)\setminus i} m^{t}_{j\rightarrow f}(x_j).
\end{align}

\emph{On a tree,} the leaves are sending the right messages at time 1 already, and after a number of time steps proportional
to the tree diameter\footnote{The eccentricity of a vertex v in a graph is the maximum distance from v to any other vertex. The diameter of a graph is the maximum eccentricity over all vertices in a graph.},
all messages are correct: the steady point is reached and the algorithm is exact. Therefore, \textbf{BP is exact on trees}. 
The JCT property discussed before is therefore useful, and can in certain cases allow us to construct graphs on which 
we know that BP is exact. However, the problem mentioned before remains: if BP is run on the JCT of a GM, subsequent 
maximizations/marginalizations will be necessary to recover the solution in terms of the initial problem's variables.

\section{Understanding Belief Propagation}
\index{Belief propagation}

We have seen how to use the (exact) elimination algorithm in order to design the BP algorithms max-product and sum-product, that are exact only 
on trees. The JCT property has taught us  how to group 
variables of an initial loopy GM such that the resulting GM is a tree (when it is possible), on which we can then run BP with a guarantee of an exact result.
However, the subsequent operations that are necessary to obtain the solution in terms of the initial problem's variables can 
be a new source of intractability.
Therefore, we would like to know what happens if we use BP on the initial (loopy) graph anyway. The advantage 
is that BP remains tractable because of the low number of operations per iteration. 
The danger is that BP is not exact anymore and therefore we need to 
ask ourselves the following 3 questions:
\begin{enumerate}
 \item Does the algorithm have fixed points?
 \item What are those fixed points?
 \item Are they reached?
\end{enumerate}
The analysis will be made with the sum-product BP algorithm, but could be carried out similarily for the max-product version.

\subsection{Existence of a fixed point}
The algorithm is of the type
\begin{equation}
 \cv{m}^{t+1} = F\left( \cv{m}^t \right) \quad {\rm with} \quad \cv{m}^t \in [0,1 ]^{2 |\mathcal{E}| |\mathcal{X} |}
\end{equation}
and \textbf{the existence of a fixed point is guaranteed} by a theorem.

\subsection{Nature of the fixed points}
Let us remind that we had factorized $P_{\cv{X}}(\cv{x}) $ in this way:
\begin{align}
 P_{\cv{X}}(\cv{x})  &\propto \prod_{i \in \mathcal{V}} \phi_i(x_i) \prod_{(i,j)\in \mathcal{E}} \psi_{ij}(x_i, x_j) \notag \\
                     &=\frac{1}{Z} e^{Q(\cv{x})}.\label{eq:pq}
\end{align}
The fixed points are a solution of the following problem
\begin{align}
P_{\cv{X}} \in\argmax_{\mu \in M( \mathcal{X}^N ) } \E_\mu \left[ Q(X) \right] + H(\mu) \label{maximization}
\end{align}
with
\begin{align}
 \E_\mu \left[ Q(X) \right] + H(\mu) = \sum_{\cv{x}} \mu(\cv{x}) Q(\cv{x})- \sum_{\cv{x}} \mu(\cv{x}) \log \mu(\cv{x}) = F(\mu).
\end{align}
Let us find a bound for this quantity. From \eqref{eq:pq}, we get $Q(\cv{x})=\log P_{\cv{X}}(\cv{x})+\log Z$. Then
\begin{align}
 F(\mu) &= \left( \sum_{\cv{x}} \mu(\cv{x}) \log Z \right) + \left( \sum_{\cv{x}} \mu(\cv{x})\log \frac{P_{\cv{X}}(\cv{x})}{\mu(\cv{x})} \right) \\
 &= \log Z + \E_{\mu} \left[ \log \frac{P_{\cv{X}}}{\mu(\cv{x})} \right] \notag \\
 &\leq \log Z + \log \E_{\mu} \left[ \frac{P_{\cv{X}}}{\mu} \right]  \quad \text{using Jensen's inequality}  \notag \\
 &\leq \log Z \notag
\end{align}
and the equality is reached when the distributions $\mu$ and $P$ are equal.\\~\\
This maximization in equation (\ref{maximization}) is made over the space of all possible distributions, which is a far too big search space. But if we restrict ourselves to trees, we know that $\mu$ has the form:
\begin{align}
 \mu \propto \prod_i \mu_i \prod_{(i,j)} \frac{\mu_{ij}}{\mu_i \mu_j}
\end{align}
BP has taught us that:
\begin{align}
 \mu_i &\propto \phi_i \prod_{k \in N(i)} m_{k \rightarrow i}  \\
 \mu_{ij} &\propto \prod_{k \in N(i) \backslash j} m_{k \rightarrow i} \, \phi_i \, \psi_{ij} \, \phi_j \prod_{l \in N(j) \backslash i} m_{l \rightarrow j}
\end{align}
If we marginalize $\mu_{ij}$ with respect to $x_j$, we should obtain $\mu_i$: $\sum_{x_j} \mu_{ij}(x_i,x_j) = \mu_i(x_i)$. Writing this out, we obtain:
\begin{align}
 \prod_{k\in N(i) \backslash j} m_{k \rightarrow i} \phi_i \left( \sum_{x_j} \psi_{ij} \phi_{j} \prod_{l \in N(i) \backslash j} m_{l \rightarrow j} \right) &= \phi_i \prod_{ k \in N(i) } m_{k \rightarrow i}
\end{align}
and this should lead us to what we believe from the fixed points of BP.
Let us make a recharacterization in terms of the fixed points. In order to lighten notations, we will write $\phi$ instead of $\log \phi$ and $\psi$ instead of $\log \psi$:
\begin{align}
 F_{\text{Bethe}} ( \mu ) &= \E_{\mu} \left[ \sum_i \phi_i + \sum_{i,j} \psi_{ij} \right] - \E_{\mu} \left[ \log \mu \right]
\end{align}
We now use following factorization
\begin{align}
 \E_{\mu} \left[ \log \mu \right] &= - \sum_i \E_{\mu_i} \left[ \log \mu_i \right] - \sum_{ij} \left( \E_{\mu_{ij}} \left[ \log \mu_{ij} \right] - \E_{\mu_i} \left[\log \mu_i \right] - \E_{\mu_j} \left[ \log \mu_j\right] \right)
\end{align}
and obtain a new expression for the Bethe free energy\index{Bethe free energy}
\begin{equation}
 F_{\text{Bethe}} = \sum_i (1 -d_i) \left( H_{\mu_i} + \E_{\mu_i}\left[ \phi_i \right] \right) + \sum_{ij} \left( H(\mu_{ij}) + \E_{\mu_{ij}} \left[ \psi_{ij} + \phi_i + \phi_j \right] \right),
\end{equation}
where $d_i$ is the degree of node $i$.

\subsubsection{Background on Nonlinear Optimization}
The problem
\begin{align}
\max_q \quad G(q)\quad \text{s.t.}\quad Aq = b
\end{align}
can be expressed in a different form by using Lagrange multipliers $\lambda$
\begin{equation}
 L(q,\lambda ) = G(q) + \lambda^T (Aq-b)
\end{equation}
and maximizing
\begin{equation}
 \max_q L(q,\lambda) = M( \lambda) \leq G(q^*) \notag
\end{equation}
\begin{equation}
 \inf_{\lambda} M(\lambda) \leq G(q^*). \notag
\end{equation}
Let us look at all $\lambda$ such that $\nabla_q L(q) = 0$. In a sense, BP is finding stationary points of this Lagrangian.

\index{Variational}
\subsubsection{Belief Propagation as a variational problem}
\index{Belief propagation}

In our case, here are the conditions we will enforce with Lagrange multipliers:
\begin{align}
 &\mu_{ij}(x_i,x_j) \geq 0 & \\
 &\sum_{x_i} \mu_{i}(x_i) = 1 \quad  &\rightarrow \lambda_i \quad\quad\quad\\
 &\sum_{x_j} \mu_{ij}(x_i,x_j) = \mu_i(x_i) \quad &\rightarrow \lambda_{j \rightarrow i}(x_i) \\
& \sum_{x_i} \mu_{ij}(x_i,x_j) = \mu_j(x_j) \quad &\rightarrow \lambda_{i \rightarrow j}(x_j)
\end{align}
The complete Lagrangian reads
\begin{align}
 \mathcal{L} = F_{\text{Bethe}} (\mu) &+ \sum_i \lambda_i \left( \sum_{x_i} \mu_i(x_i) -1 \right) \notag \\
          &+ \sum_{ij} \left [ \left( \sum_{x_j} \mu_{ij}(x_i,x_j) - \mu_i(x_i) \right) \lambda_{j \rightarrow i} (x_i) \right. \notag \\
    &+ \left. \left( \sum_{x_i} \mu_{ij}(x_i,x_j) - \mu_j(x_j) \right) \lambda_{i \rightarrow j} (x_j) \right ].
\end{align}
We need to minimize this Lagrangian with respect to all possible variables, which we obtain by setting the partial derivatives to zero:
\begin{align}
 \frac{ \partial \mathcal{L}}{\partial \mu_i(x_i)} &= 0 \\
 &= - (1-d_i)(1+\log \mu_i(x_i)) + (1-d_i) \phi_i(x_i) + \lambda_i - \sum_{j \in N(i)} \lambda_{j \rightarrow i} (x_i) \notag
\end{align}
which imposes following equality for the distribution $\mu_i$:
\begin{equation}
 \boxed{ \mu_i(x_i) \propto e^{\phi_i(x_i) + \frac{1}{d_i-1} \sum_{j\in N(i)} \lambda_{j \rightarrow i}(x_i)} }
\end{equation}
Let us now use the transformation $\lambda_{j \rightarrow i} (x_i) = \sum_{k \in N(i) \backslash j} \log m_{k \rightarrow i} (x_i) $, and we obtain
\begin{equation}
  \sum_{j \in N(i)} \lambda_{j \rightarrow i} (x_i) \equiv (d_i -1) \sum_{j \in N(i)} \log m_{j \rightarrow i}(x_i).
\end{equation}
In the same way, we can show that:
\begin{equation}
 \frac{ \partial \mathcal{L}}{\partial \mu_{ij}(x_i,x_j)} = 0 \quad \Rightarrow \boxed{ \mu_{ij}(x_i,x_j) \propto e^{\phi_i(x_i) + \phi_j(x_j) + \psi_{ij}(x_i,x_j) + \lambda_{j \rightarrow i}(x_i) + \lambda_{i \rightarrow j} (x_j)}} \nonumber
\end{equation}
This way, we found the distributions $\mu_i$ and $\mu_{ij}$ that are the fixed points of BP.

\subsection{Can the fixed points be reached?}
We will now try to analyze if the algorithm can actually reach those fixed points that we have exhibited in the previous section.
Let us look at the simple (but loopy) graph in Fig. \ref{fig:fig8}.
\begin{figure}[h!]
\begin{center}
 \includegraphics[width=0.3\textwidth]{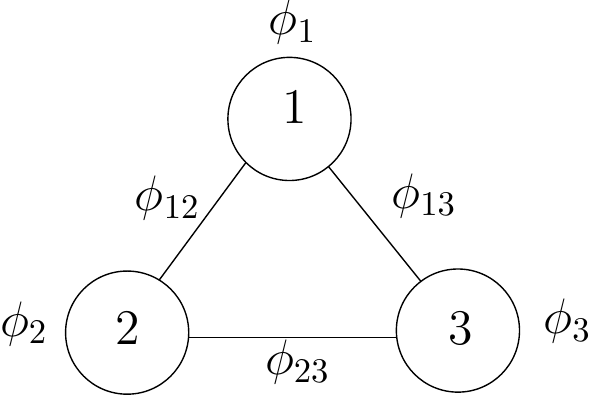}
\end{center}
 \caption{A simple loopy graph.}
 \label{fig:fig8}
\end{figure}
At time $t=1$, we have
\begin{equation}
 m_{2\rightarrow 1}^1 (x_1) \propto \sum_{x_2} \phi_2(x_2) \phi_{12}(x_1,x_2) \underbrace{m_{3 \rightarrow 2}^0 (x_2)}_{=1}
\end{equation}
and 
\begin{equation}
 m^1_{3 \to 1} \propto \sum_{x_3} \phi_3 \phi_{13}
\end{equation}
which also corresponds to the messages of the modified graph in Fig. \ref{fig:fig9}.
\begin{figure}[h!]
\begin{center}
  \includegraphics[width=0.3\textwidth]{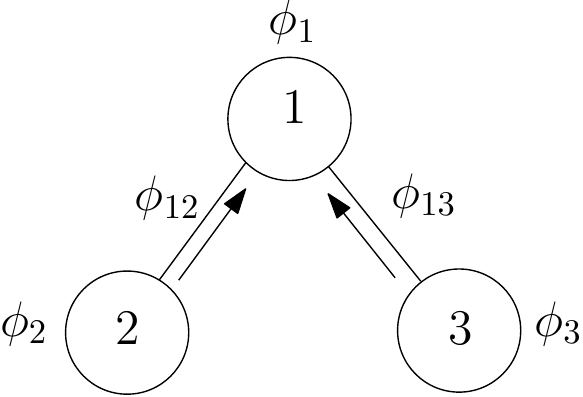}
\end{center}
\caption{Graph seen by BP at time $t=1$.}
\label{fig:fig9}
\end{figure}
\begin{figure}[h]
\begin{center}
  \includegraphics[width=0.3\textwidth]{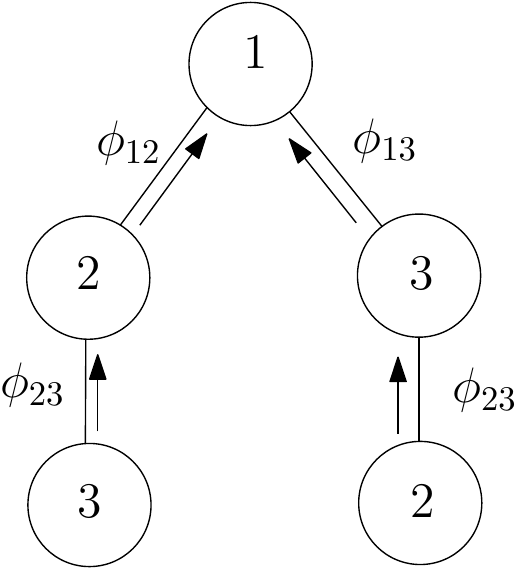}
\end{center}
\caption{Graph seen by BP at time $t=2$.}
\label{fig:fig10}
\end{figure}

At time $t=2$, the messages will be as
\begin{equation}
 m_{2 \rightarrow 1}^{2} \propto \sum_{x_2} \phi_2 \phi_{12} m_{3 \rightarrow 2}^{1}(x_2)
\end{equation}
corresponding to the messages on the modified graph in Fig. \ref{fig:fig10}.
If we increase $t$, the corresponding non-loopy graph gets longer at each time step.\\
Another way of seing this is by looking at the recursion equation:
\begin{align}
 F_{ij} (m^*) &= m_{ij}^* \\
 m_{ij}^{t+1} &= F_{ij} (m^t) \notag \\
 | m_{ij}^{t+1} - m_{ij}^* | &= |F_{ij}(m^t) - F_{ij}(m^*) | \notag \\
   &= | \nabla F_{ij} (\theta)^T (m^t -m^*) | \quad \text{(mean value theorem)}  \notag \\
   |m^{t+1} - m^*|_{\infty} &\leq | \nabla F_{ij} (\theta)|_{1} |m^t -m^*|_{\infty}
\end{align}
From this last inequality, it is clear that if we can prove that $|F_{ij}|_1$ is bounded by some constant $\rho < 1$, the 
convergence is proved. Unfortunately, it is not often easy to prove such a thing. 

\subsubsection{The hardcore model}
In the hardcore model, we have
\begin{align}
 \phi_i(x_i) &= 1 \quad \text{for all } \, x_i \in \{ 0,1 \} \\
 \psi_{ij}(x_i,x_j) &= 1 - x_i x_j.
\end{align}
\noindent
Instead of using BP, let us do the following gradient-descent like algorithm:
\begin{equation}
 y(t+1) = \left[ \left. y(t) + \alpha(t) \frac{\partial F}{\partial y_i}\right|_{y(t)}  \right]
\end{equation}
where the operator $\left[ . \right]$ is a clipping function that ensures that the result stays in the interval $(0,1)$.
This is a projected version of a gradient algorithm with variable step size $\alpha(t)$. Choosing this step size with following rule:
\begin{equation}
 \alpha(t) = \frac{1}{\sqrt{t}}\; \frac{1}{2^d}
\end{equation}
then we can show that in a time $T \sim n^2 2^d \frac{1}{\epsilon^4}$ we will find $F_b$ up to $\epsilon$, and convergence is proved.


\section{Learning Graphical Models}
\index{Learning}
\index{Graphical model}

In this final section, we focus on the learning problem. In particular, we consider three different cases:
\begin{itemize}
\item Parameter learning\\
Given a graph, the parameters are learned from the observation of the entire set of realizations of all random variables.
\item Graphical model learning\\
Both the parameters and the graph are learned from the observations of the entire set of realizations of all random variables.
\item Latent graphical model learning\\
The parameters and  the graph are learned from partial observations: some of the random variables are assumed to be hidden.
\end{itemize}

\subsection{Parameter learning}
 \subsubsection{Single parameter learning}
 We consider the following simple setting where $x_i$ is a Bernoulli random variable with parameter $\theta$:
 \begin{align}
 P_X(x_i,\theta)&=
\left\lbrace 
	\begin{array}{ll}
	\theta & \text{if $x_i=1$},\\
	1-\theta & \text{if $x_i=0$}.
	\end{array}
	\right.
 	\label{Bernoulli}
 \end{align}
\index{MAP}
Given observations $\lbrace x_1,\ldots, x_S \rbrace$, we are interested in the MAP estimation of the parameter $\theta$:
 \begin{align}
\hat{ \theta}^{MAP}
 &= \argmax_{\theta\in [0,1]} P(\theta|x_1,\ldots x_S),\nonumber\\
 &= \argmax_{\theta\in [0,1]} P(x_1,\ldots x_S|\theta)\;p(\theta),
 \end{align}
 where maximizing $P(x_1,\ldots x_S|\theta)$ leads to the maximum likelihood (ML) estimator $\hat{\theta}^{ML}$ of $\theta$.\\
 Denoting $\mathcal{D}\triangleq\lbrace x_1,\ldots x_S \rbrace$ the observed set  of realizations, we define the empirical likelihood 
  as follows:
 \begin{align}
 \ell(\mathcal{D};\theta)
 &= \frac{1}{S} \log P(x_1,\ldots x_S|\theta),\nonumber\\
 &= \frac{1}{S} \sum_i \log P(x_i|\theta),\nonumber\\
 &=\hat{P}(1)\;\log \theta + \hat{P}(0)\;\log (1-\theta),\label{eq:thetaest}
 \end{align}
 with $\hat{P}(1) = \frac{1}{S}\sum_i^S\mathbb{1}_{\lbrace x_i=1\rbrace}$. Derivating \eqref{eq:thetaest} and setting the result to zero, we obtain
the \textbf{maximal likelihood estimator} $\hat{\theta}^{ML}$: 
 \begin{align}
& \frac{\partial}{\partial \theta} \ell(\mathcal{D};\theta) = \frac{\hat{P}(1)}{\theta}-\frac{\hat{P}(0)}{1-\theta}= 0, \nonumber \\
 \Rightarrow \quad\quad&\quad\quad\quad\quad\boxed{\hat{\theta}^{ML}=\hat{P}(1)}
 \end{align}
 What is the amount of samples $S$ needed to achieve $\hat{\theta}^{ML}(S) \approx (1\pm \epsilon)\theta$? Considering the binomial variable
 $B(S,\theta)$ (which is the sum of $S$ independently drawn Bernoulli variables from (\ref{Bernoulli})), we can write
 \begin{align}
 &P(|B(S,\theta)-S\theta|>\epsilon S\theta)\sim \exp(-\epsilon^2 S\theta)\leq \delta,\nonumber\\
 \Rightarrow \quad\quad&\quad\quad\quad\quad\boxed{S\geq \frac{1}{\theta}\frac{1}{\epsilon^2}\log \frac{1}{\delta}}
 \end{align}
 
 \subsubsection{Directed graphs}
 We consider the following setting in which we have not one, but many random variables to learn on a directed graph:
 \begin{align}
 P_{\underline X}(\underline{x})\propto\prod_i P_{X_i|X_{\Pi_i}}(x_i|x_{\Pi_i}), 
 \end{align}
 where $\Pi_i$ stands for the parents of node $i$, and $P_{X_i|X_{\Pi_i}}(x_i|x_{\Pi_i})\triangleq \theta_{x_i,x_{\Pi_i}}$.\\
Again, we look at the empirical likelihood
\begin{align}
\ell(\mathcal{D};\underline{\theta})
&=\sum_i \sum_{x_i,x_{\Pi_i}} \hat{P}(x_i,x_{\Pi_i})\log  \theta_{x_i,x_{\Pi_i}},\nonumber\\
&=\sum_i \sum_{x_i,x_{\Pi_i}} \hat{P}(x_i|x_{\Pi_i})\hat{P}(x_{\Pi_i}) \left[\log \frac{\theta_{x_i,x_{\Pi_i}}}{\hat{P}(x_i|x_{\Pi_i})} + \log \hat{P}(x_i|x_{\Pi_i})\right], \nonumber\\
&=\sum_i \sum_{x_i,x_{\Pi_i}} \hat{P}(x_i|x_{\Pi_i})\hat{P}(x_{\Pi_i}) \log \frac{\theta_{x_i,x_{\Pi_i}}}{\hat{P}(x_i|x_{\Pi_i})},
\end{align}
and set the derivative to zero in order to obtain the ML estimation of $\underline{\theta}$, resulting in
\begin{align}
&\sum_{x_i} \hat{P}(x_i|x_{\Pi_i})\log \frac{\theta_{x_i,x_{\Pi_i}}}{\hat{P}(x_i|x_{\Pi_i})}=
\E_{\hat{P}} \left[\log \frac{\theta_{x_i,x_{\Pi_i}}}{\hat{P}(x_i|x_{\Pi_i})}\right], \notag \\
 \Rightarrow \quad\quad&\quad\quad\quad\quad\boxed{\hat{\theta}_{x_i,x_{\Pi_i}}^{ML}=\hat{P}(x_i|x_{\Pi_i})}
 \end{align}
 
 \subsubsection{Undirected graphs}\index{Undirected}
 Let us now consider the case of undirected graphs.
 To reduce the amount of indices, we will write $i$ instead of $x_i$ in the following.
 \begin{align}
  &\text{On a tree, } &P_X= \prod_i P_i \prod_{ij} \frac{P_{ij}}{P_i P_j} & \to \text{ possible estimator: } \hat{P}_i \frac{\hat{P}_{ij}}{\hat{P}_i \hat{P}_j} \nonumber \\
  &\text{on a chordal graph, } &P_X \propto \frac{\prod_C \phi_C(x_C)}{\prod_S \phi_S(x_S)} & \to \text{ possible estimator: } \frac{\hat{P}_C}{\hat{P}_S} \nonumber \\
  &\text{on a triangle-free graph, } &P_X \propto \prod \phi_i \prod_{ij} \psi_{ij} \nonumber
 \end{align}
For the last case, let us use the Hammersley-Clifford theorem. Let $\mathcal{X} = \{ \text{0}, \text{1} \}$. On a triangle-free graph,
the maximal clique size is 2, and therefore we can write
\begin{equation}
 P_{\cv{X}}(\cv{x}) \propto \exp \left( \sum_i U_i(x_i) + \sum_{ij} V_{ij}(x_i,x_j) \right).
\end{equation}
Using the fact that we have a MRF, we get
\begin{align}
 \frac{P(X_i=1,X_{rest} = \mathbb{0})}{P(X_i=0,X_{rest} = \mathbb{0})} \propto \exp \left(  Q({i}) \right).
\end{align}
Also, because of the fact that on a MRF, a variable conditioned on its neighbours is independent of all the others, we can write
\begin{align}
 \frac{P(X_i=1,X_{rest} = \mathbb{0})}{P(X_i=0,X_{rest} = \mathbb{0})} = \frac{P(X_i=1,X_{N(i)} = \mathbb{0})}{P(X_i=0,X_{N(i) } = \mathbb{0})}
\end{align}
and therefore this quantity can be calculated with $2^{|N(i)| + 1 }$ operations.

 \subsection{Graphical model learning}
What can we learn from a set of realizations of variables when the underlying graph is not known? We focus now in the following maximisation
 \begin{align}
\max_{\mathcal{G}, \theta_\mathcal{G }} \ell(\mathcal{D};\mathcal{G}, \theta_\mathcal{G}) =\max_{\mathcal{G}} \underbrace{\max_{\theta_\mathcal{G }} \ell(\mathcal{D};\mathcal{G}, \theta_\mathcal{G})}_{\hat{\ell}(\mathcal{D};\mathcal{G}  ) \triangleq\ell(\mathcal{D}; \mathcal{G} , \hat{\theta}_{\mathcal{G}}^{ML}) }.
 \end{align}
 From the previous subsection, we have $\hat{\theta}^{ML}_\mathcal{G}$, and therefore we only need to find a way to evalute the maximization on 
 the possible graphs.
  
  \subsubsection{Directed graphs}
 On a directed graph $\mathcal{G}\rightarrow (i,\Pi_i)$, the empirical likelihood reads
\index{Directed}
 \begin{align}
 \hat{\ell}(\mathcal{D};\mathcal{G})
 &=\sum_i \sum_{x_i,x_{\Pi_i}} \hat{P}(x_i,x_{\Pi_i})\log \hat{P}(x_i|x_{\Pi_i}),\nonumber\\
 &=\sum_i \sum_{x_i,x_{\Pi_i}} \hat{P}(x_i,x_{\Pi_i})\log\left[ \frac{\hat{P}(x_i,x_{\Pi_i})}{\hat{P}(x_i)\hat{P}(x_{\Pi_i})} \hat{P}(x_i)\right], \nonumber\\
 &=\sum_i \sum_{x_i,x_{\Pi_i}} \hat{P}(x_i,x_{\Pi_i})\log\frac{\hat{P}(x_i,x_{\Pi_i})}{\hat{P}(x_i)\hat{P}(x_{\Pi_i})} +\sum_{x_i}\hat{P}(x_i)\log \hat{P}(x_i),\nonumber \\
  &=\sum_i I(\hat{X}_i;\hat{X}_{\Pi_i})- H(\hat{X}_i).
 \end{align}

 Looking for the graph maximizing the empirical likelihood thus consists in maximising the mutual information: $\max_\mathcal{G} \sum_i I(\hat{X}_i;\hat{X}_{\Pi_i})$. 
 In a general setting, this is not easy. \textbf{Reducing the search space to trees} however, some methods exist, like the Chow-Liu algorithm \cite{CL1968}, which relies on the procedure used to get the maximum weighted spanning tree (cf. section 2).
 
 \subsubsection{Undirected graphs}\index{Undirected}
What can we do in the case of undirected graphs? 
Let us restrict ourselves to the binary case $\cv{x} \in \{ {\rm 0},{\rm 1} \}^N $ and to exponential families:
 \begin{align}
 P_{\underline X}(\underline{x})=\exp\left(\sum_i \theta_i x_i+\sum_{i,j}\theta_{ij}x_i x_j - \log Z(\underline{\theta})\right).
 \end{align}
Again, we denote $\mathcal{D} = \{ \cv{x}^1, \cdots , \cv{x}^S \}$ the observed dataset, and 
 the log-likelihood can be written as
\begin{align}
\ell(\mathcal{D};\underline{\theta})= \underbrace{\sum_i \theta_i \mu_i + \sum_{i,j} \theta_{ij} \mu_{ij}}_{\langle \theta, \mu \rangle} - \log Z(\underline{\theta}).
\end{align}
As $\ell(\mathcal{D};\underline{\theta})$ is a concave function of $\underline{\theta}$, it can be efficiently solved using a gradient descent algorithm of the form
\begin{align}
\boxed{\underline{\theta}^{t+1}=\underline{\theta}^{t} + \alpha(t) \nabla_{\underline{\theta}} \ell(\mathcal{D};\underline{\theta})|_{\underline{\theta}=\underline{\theta}^t}}
\end{align}
The difficulty in this formula is the evaluation of the gradient:
\begin{align}
  \nabla_{\underline{\theta}} \ell(\mathcal{D};\underline{\theta}) = \mu - \E_{\underline{\theta}} ( \cv{X} ),
\end{align}
whose second term is an expectation that has to be calculated, using the sum-product algorithm or with a Markov chain Monte Carlo method for instance.

Another question is whether we will be learning interesting graphs at all. Graph-learning algorithms tend to link variables that are not linked in the real
underlying graph. To avoid this, complicated graphs should be penalized by introducing a regularizer. Unfortunately, this is a highly non-trivial problem, and 
graphical model learning algorithms do not always perform well to this day.

 \subsection{Latent Graphical Model learning: the Expectation-maximization algorithm}
\index{Expectation maximization}
In this last case, we distinguish two different variables: 
\begin{itemize}
\item $Y$ stands for observed variables,
\item $X$ denotes the hidden variables.
\end{itemize}
The parameter ${\theta}$ is estimated from the observations, namely
\begin{align}
\hat{\theta}^{ML}=\argmax_{\theta} \log P_Y(y;\theta).
\end{align}
The log-likelihood is derived by marginalizing on the hidden variables
\begin{align}
\ell(y;\theta)
&= \log P_Y(y;\theta),\nonumber\\
&=\log \sum_x P_{X,Y} (x,y;\theta),\\
&=\log \sum_x q(x|y) \frac{P_{X,Y} (x,y;\theta)}{q(x|y)},\\
&= \log \E_{q} \left[\frac{P}{q}\right]\geq  \E_{q} \left[\frac{P}{q}\right] \triangleq \mathcal{L}(q;\theta).
\end{align}

\noindent
This gives raise to the Expectation-Maximisation (EM) algorithm \cite{DLR1977}.
\begin{algorithm}
\textbf{EM algorithm}\\
Until convergence, iterate between
 \begin{itemize}
\item \textbf{E-step}: estimation of the distribution $q$\\
$\theta^t\rightarrow q^{t+1}=\argmax_q \mathcal{L}(q;\theta^t)$.
\item \textbf{M-step}: estimation of the parameter $\theta$\\
$q^{t+1}\rightarrow \theta^{t+1}=\argmax_\theta \mathcal{L}(q^{t+1};\theta)$.
\end{itemize}
\end{algorithm}

\newpage
\bibliographystyle{plain}
\addcontentsline{toc}{section}{References}
\bibliography{references}

\begin{thebibliography}{1}

\bibitem{CL1968}
C.~K. Chow and C.~N. Liu.
\newblock Approximating discrete probability distributions with dependence
  trees.
\newblock {\em Information Theory, IEEE Transactions on}, 14(3):462--467, 1968.

\bibitem{DLR1977}
A.~P. Dempster, N.~M. Laird, and D.~Rubin.
\newblock Maximum likelihood from incomplete data via the em algorithm.
\newblock {\em Journal of the Royal Statistical Society, Series B},
  39(1):1--38, 1977.

\bibitem{Gri1973}
G.~R. Grimmet.
\newblock A theorem about random fields.
\newblock {\em Bulletin of the London Mathematical Society}, 5(1):81--84, 1973.

\bibitem{HC1971}
J.~M. Hammersley and P.~Clifford.
\newblock Markov fields on finite graphs and lattices.
\newblock Available online:
  \url{http://www.statslab.cam.ac.uk/~grg/books/hammfest/hamm-cliff.pdf}, 1971.

\end{thebibliography}

\end{document}